\documentclass{article}

\usepackage{latexsym,amssymb}
\usepackage[utf8]{inputenc} % allow utf-8 input
\usepackage[T1]{fontenc}    % use 8-bit T1 fonts
\usepackage{url}            % simple URL typesetting
\usepackage{booktabs}       % professional-quality tables
\usepackage{amsfonts}       % blackboard math symbols
\usepackage{nicefrac}       % compact symbols for 1/2, etc.
\usepackage{microtype}      % microtypography
\usepackage{graphicx}
\usepackage{color}

\date{}

\title{Subspace Methods That Are Resistant\\
  to a Limited Number of Features Corrupted by an Adversary}

\author{Chris Mesterharm, Rauf Izmailov, Scott Alexander, and Simon Tsang\\
  Perspecta Labs, Basking Ridge, NJ 07920\\
}

\newtheorem{lemma}{Lemma}
\newtheorem{theorem}{Theorem}
\newtheorem{corollary}{Corollary}

\newenvironment{proof}
{\noindent \textbf{Proof}: \begin{quote} \vspace{-.1in}}
  {~$\Box$ \end{quote}}

\begin{document}

\maketitle

\begin{abstract}
  In this paper, we consider batch supervised learning where an
  adversary is allowed to corrupt instances with arbitrarily large
  noise. The adversary is allowed to corrupt any $l$ features in each
  instance and the adversary can change their values in any way.  This
  noise is introduced on test instances and the algorithm receives no
  label feedback for these instances.  We provide several subspace
  voting techniques that can be used to transform existing algorithms
  and prove data-dependent performance bounds in this setting.  The
  key insight to our results is that we set our parameters so that a
  significant fraction of the voting hypotheses do not contain corrupt
  features and, for many real world problems, these uncorrupt
  hypotheses are sufficient to achieve high accuracy.  We empirically
  validate our approach on several datasets including two new
  datasets that deal with side channel electromagnetic information.
\end{abstract}

\section{Introduction} \label{Introduction}

In this paper, we consider standard batch supervised learning with the
additional assumption that an adversary can corrupt a subset of
features during the evaluation/use of the algorithm's hypothesis.  We
give techniques to make any machine learning algorithm resistant to
this corruption and give data dependent upper bounds on the error
rate.  This type of analysis is useful since it gives a way to weaken
the independent and identically distributed (iid) assumption of many
machine learning results.  Many practical problems cannot be
accurately modeled with the iid assumption.

The feature corruption we consider is related to the popular
adversarial research that attempts to change the predicted label of an
image without significantly changing the image or its true label as
confirmed by a human~\cite{SZS13}.  While our techniques can be
applied to that setting, we do not focus on cases where the adversary
cannot change the true label.  If the structure of the problem is such
that the adversary can corrupt the instance and change the label then
our bounds will reflect the increase in error rate.  While this issue
is unavoidable for general machine learning problems, we find that
many realistic machine learning problems have large amount of
redundancy which can be exploited to help tolerate this type of
corruption.

Our corruption model allows the adversary to change up to $l$ features
in any possible way.  This includes picking different features to
modify for each instance.  This is an established adversarial
corruption model \cite{CW17} and is useful for more than just images.
For example, any problem that has categorical features should allow
arbitrary changes to the features since there is no natural metric on
features to control the amount of change.  Another example is missing
values.  While some missing values might occur during training, an
increase during production use can be modeled by an adversary and our
techniques can give performance bounds in this setting.

One thing we wish to stress is that while our techniques can work
against a strong adversary, we feel its primary benefit is situations
that are not adversarial but also not iid.  For example, consider the
popular example of autonomous driving.  Even if there are no
adversarial instances on the road, we still want a system that has
provable guarantees that are stronger than traditional iid
assumptions.  If the car is driven in a non-training environment or a
sensor starts to malfunction, we want the system to still make
accurate predictions or at least report that something is wrong.
This is related to redundant and error correcting systems, but in this
context, we are building it into the machine learning classifier.

The main intuition of our technique is to use majority vote where each
hypotheses only uses a fraction of the features such that we can
guarantee that more than half of the hypotheses do not use any corrupt
features.  We give a range of methods and show that we can tolerate
$O(n/k)$ corrupt features where $n$ is the total number of features
and $k$ is the number of features in each hypothesis.  Our main result
is to give a data dependent bound for our techniques.  This works in
the same way as a standard test set and can be used to generate
upper bounds on the error rate for different amounts of corruption.
The technique is broadly applicable as one can use validation sets to
pick an effective $k$ value against a worst case adversary.  While in
most cases, it is not possible to know how many features the adversary
can corrupt, we find that the validation set has a sweet spot where
the accuracy gains from a higher $k$ value are offset by the fact that
the adversary can corrupt more hypotheses as $k$ increases.

Our techniques can be used with any machine learning algorithm.  After
picking $h$ subsets of features one just applies the machine learning
algorithm(s) to each subset to create the majority vote classifier.
One advantage of this approach is that our method is able to inherit
many of the benefits of the basic algorithm used to learn the
hypotheses.  In some cases, this could include resistance to different
types of adversarial attacks.  For example, if there is an algorithm
that can tolerate the popular adversarial image attack where each
pixel can change a small amount, we can use that algorithm as the
basic algorithm to generate a classifier that can tolerate examples
generated by a more powerful adversary that can arbitrarily modify $l$
features and make minor changes to all $n$ features.

There has been significant amount theoretical research in concept
drift \cite{SK12} and adversarial learning~\cite{MMSTV18, FFF18}.  Our
techniques are strongly related to the random subspace
method~\cite{Ho98}.  In fact, this is one of the algorithms we use in
our experiments.  Our main contribution over the existing research is
to analyze the random subspace method in the adversarial setting and
to give new subspace methods that have stronger adversarial
performance bounds.  In \cite{QSSL09}, a SVM-based optimization
problem is given based on the assumption that features are removed by
an adversary giving them a value of zero.  Our adversary is more
general in that it can generate any value for the corrupt features
allowing it to greatly distort the prediction or hide the
modifications.  In \cite{BFHP12}, a similar strategy of using few
features per basic hypotheses is presented, but the technique uses a
stacking approach that can give larger influence to hypotheses that
might have been corrupted.  They show experimentally that their
technique is effective with their weaker adversarial assumptions. In
\cite{BFR10}, intuitive arguments are given for ensemble methods,
including the random subspace method, but again the results are
experimental.

There has also been a large amount of research on this problem in the
online setting where constant label feedback is available and can be
used to adapt the hypotheses to changes in the target
function~\cite{LW94, HW01, Mes03}. Our analysis is for the more
difficult problem where no label feedback is available after training
and our approach is to construct a fixed hypothesis that is robust to
the adversarial changes.

\section{Adversarial Learning Problem} \label{problem}

Let $X = \mbox{R}^n$ be the instance space and $y = \{1,\ldots,d\}$ be
be a set of discrete labels.  Let $\mbox{Train}$ be a sequence of
training instances selected independently from distribution
$\mbox{P}(X,y)$ and let
$\mbox{Test} = \langle e_i | i \in {1,\ldots,m}\rangle$ be an sequence
of test instances selected independently from the same distribution.
The adversary is allowed to arbitrarily corrupt $l$ different features
on every test instance.  More formally, let
$C = \langle c_i | i \in {1,\ldots,m}\rangle$ be this corrupted
sequence of instances where for all $i$, $||(e_i-c_i)||_0 \leq l$.

To help describe our algorithms we use the term {\em corrupt feature}
to refer to any feature that has been modified by the adversary and
{\em corrupt hypothesis} to refer to any hypothesis that contains a
corrupt feature.  For voting, we refer to any hypothesis used in the
vote as a {\em basic} hypothesis.  As mentioned, the testing has a
potentially infinite number of instances and we refer to each one as a
{\em trial}.  While the term trial is often used in online learning
\cite{Lit88}, we should stress our results are not for the online
model as we do not receive label feedback after making a prediction.
Instead, we learn a static classifier that is robust to adversarial
changes in the instances.

\section{Majority Vote Data Dependent Bounds} \label{Data-Dependent}

The key component of our technique is majority vote where each basic
hypotheses predicts a single label and the voting prediction is the
label that occurs most frequently and randomly over the labels in the
case of a tie.  The main intuition is that we generate basic
hypotheses for the voting such that a majority of them do not contain
features that are corrupt.  If these uncorrupt hypotheses predict the
correct label than the predicted label will be correct.  Of course, it
is unlikely that these uncorrupt hypotheses will be perfect, so in
this section we give a data dependent bound on the error rate using
uncorrupt test data.  In the next section, we will show how to control
the number of corrupt hypotheses as a function of the number of
corrupt features, $l$.

To state our main result, we start with some definitions.  Let $s_x$
be the vector of label counts made by the majority vote on instance
$x$.  Define $\delta(s,y)=s[y] - \max_{y' \not = y}s[y']$ where $s$ is
an integer vector with $d$ elements; one element for each label value.
Notice that $\delta(s_x,y)$ corresponds to the margin of instance
$(x,y)$ with respect to the majority vote.  Define the loss as
$l_c(s,y) = I(\delta(s,y)-2c \leq 0)$ where $I$ is the indicator
function and $c$ will depend on properties of the adversary.  We will
show this loss function can be used to upper bound the error rate on
instances that have been corrupted by the adversary.

\begin{theorem} \label{concentration}
  Assume that an adversary can corrupt at most $c$ majority vote
  hypotheses on any instance.  Let
  $T=\langle(x_1,y_1),\ldots,(x_m,y_m)\rangle$ be a sequence of $m$
  uncorrupted test instances sampled from distribution $P(x,y)$.  The
  error rate of the majority vote on corrupt instances is greater than
  $\sum_{i=1}^{m} l_c(s_{x_i},y_i)/m + \epsilon$ with
  probability at most $\exp\left(-2m\epsilon^2\right)$.
\end{theorem}

\begin{proof}
  Let $(x_i,y_i)$ be an instance generated by the adversary on trial
  $i$ and let $s_{x_i}$ be the vector of label counts.  The margin for
  this instance is $M_i = \delta(s_{x_i},y_i)$.  We can We can
  decompose this margin into two components as $M_i = X_i - B_i$ where
  $X_i$ is margin that results from selecting the instance from
  distribution $P(x,y)$ and $B_i$ is based on the adversary corrupting
  $c$ hypotheses.  Since every corrupt hypotheses can decrease the
  margin by at most 2 by shifting one vote away from the correct label
  and moving it to the incorrect label with the highest vote, we have
  $M_i \geq X_i - 2c$.  Since $X_i - 2c$ is a random variable, we can
  apply a Hoeffding bound to the Bernoulli variables
  $I(X_i - 2c \leq 0)$.  This bound will also be valid for
  $I(M_i \leq 0)$, since $M_i \geq X_i - 2c$.  A direct application
  Theorem~1 in \cite{Hoe63} proves that the empirical mean of
  $I(X_i - 2c \leq 0)$ has a probability of at most
  $\exp\left(-2\epsilon^2m\right)$ of exceeding the true mean by
  $\epsilon$.
\end{proof}

It is easiest to understand this bound with a plot.  In
Figure~\ref{score-plot}, we give a histogram of the Score function
over all the test data.  Each corrupt hypothesis will shift this
histogram by at most 2 to the left.  By counting all the values that
are still above 0, we get a bound on the error rate of majority vote.
In this case, the test error rate is around 0.06, but the numbers in
this plot can be used to show that in the worst case the error is at
most 0.1 when 5 features are corrupt.

\begin{figure}[ht]
  %\vskip -0.1in
  \begin{center}
    \centerline{\includegraphics[scale=0.3]{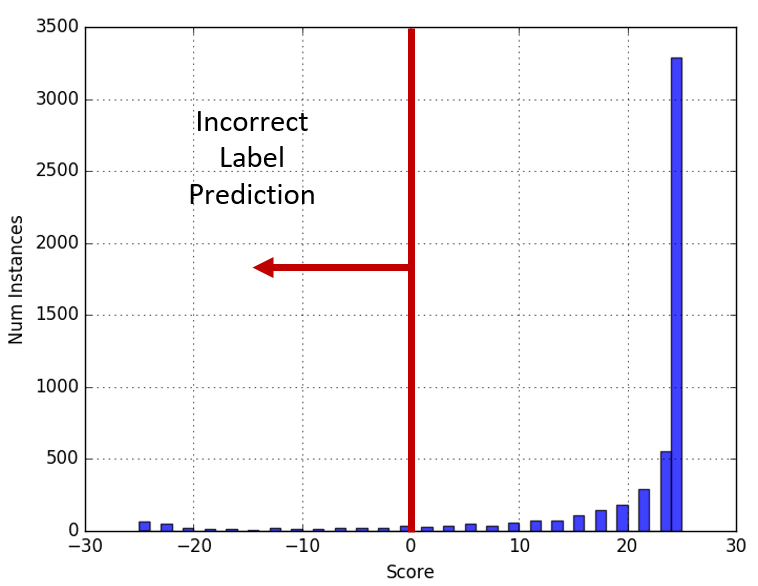}}
    \caption{Histogram of Score function values on 5248 test
      data points with 25 hypothesis for UCI character font
      dataset. The concept is to determine if character is italic with
      Arial font. For this experiment, a corrupt feature corrupts at
      most one basic voting hypothesis.}
    \label{score-plot}
  \end{center}
  \vskip -0.25in
\end{figure}

Notice that the histogram does not show the somewhat idealized case of
independence of errors, but for our purpose, dependence is fine.
Intuitively, what we are exploiting is redundancy of features.  While
the different hypotheses learned with these redundant features might
be highly correlated, we assume the corruption of one feature has no
effect on any related redundant features.  Also, while it might seem
restrictive to have a hard limit on $l$, and therefore $c$, it is
straightforward to assume $l$ comes from a distribution and use this
distribution to create an upper-bound on the error.

The bound in Theorem~\ref{concentration} can be improved by
incorporating information about the variance of the random variables,
but we omit the details for clarity and space. In addition the result
can be improved slightly by taking into account that the majority vote
makes a random prediction on ties. In practice, just as for normal
test set bounds, a computer algorithm should be used to get precise
bounds on the binomial distribution~\cite{Langford05}.  Also notice
that this bound is worst case in that it assumes a corrupt hypotheses
always makes the wrong prediction in the worst way possible.  This is
a reasonable assumption as it is true for certain types of machine
learning problems and classifiers.  For example, hyperplane
classifiers with $X=R^n$.  As mentioned in the introduction, the bound
can be improved by making more assumptions about the basic learning
algorithms used to generate the ensemble.

\section{Majority Vote Hypotheses Generation} \label{algorithms}

In this section, we give four techniques to generate hypotheses for
the majority vote.  They are all subspace methods since they only
select a subset of features for each hypotheses.  Our goal is to bound
how many of these $h$ hypotheses can be corrupted by an adversary that
is allowed to corrupt at most $l$ features on each instance.  To help
explain our results, let $n$ be the total number of features and let
$r$ be the percentage of hypotheses that are corrupt.  For all
methods, we will show that in order to have more than half the
hypotheses uncorrupted, in the worst case, they can at most tolerate
$O(n/k)$ corrupt features.  Roughly speaking, if we want to double the
number of corrupt features the algorithm can tolerate, we need the
halve the number of features in each hypothesis.

For all methods, we suggest, at a minimum, to randomize the initial
ordering of the features, since feature values and relevance might be
correlated in the ordering.  Another option is to use some type of
feature selection procedure, such as mutual information or domain
knowledge in an attempt to equalize the quality of the features across
the hypotheses.  This includes using validation sets to evaluate
possible feature orderings or parameter setting of the methods.

\subsection{Fixed Feature Split} \label{Fixed-Split}

We call our first method is called fixed-split as it partitions the
features into approximately equal sized disjoint groups. Our only
parameter is the number of hypotheses, $h$, and we attempt to
partition the features as evenly as possible.  If there is a remainder
when dividing $n$ by $h$, the remainder will be split by adding one to
the some of the hypotheses.  This means each hypotheses will have
either $k=\lfloor n/h \rfloor$ or $k=\lfloor n/h \rfloor + 1$.  To
simplify the analysis, we will assume $k=n/h$ as this does not
significantly change the results.

Notice that each corrupt feature will corrupt at most one hypothesis.
For example, with 900 features, we could learn 9 hypotheses that each
have 100 features. In terms of our variables, we have
$r \leq l/h = lk/n$.  This is not a strict equality since multiple
corrupt features might occur in a single hypothesis.  It is useful for
comparison with the other methods to rearrange this as $l \geq rn/k$
which shows how the tolerance to corrupt features is proportional to
$n/k$.

\subsection{All Size $k$ Feature Subsets} \label{Subset-Algorithm}

In this section, we consider the technique of generating every
possible $k$ feature hypothesis.  We call this the $k$-subset method.
While this can be intractable, it is a useful comparison case as all
of the remaining method are related to this technique.

Since we are considering every way to select $k$ features, there are a
total of $h=C(n,k)$ hypotheses.  Since a corrupt hypothesis has one or
more corrupt features, there are total of $C(n-l,k)$ uncorrupt
hypotheses, and therefore,
\[ r = 1 - \frac{C(n-l,k)}{C(n,k)} =
  1 - \prod_{i=0}^{i=l-1} \frac{n-k-i}{n-i}. \]
To make this result easier to interpret, we can use the fact that for
$x \in (0,1]$, $ln(1-x) \leq 2x/(2+x)$~\cite{Topsoe95} to prove that
for $0<a<x$, $\ln(x)-\ln(x-a) > 2a/(2x-a)$.  Using our previous
equality on $r$, we get $l < \ln(1/(1-r))n/k - 1/2$.  What we really
want is a lower bound on $l$, but we found it difficult to achieve a
tight and simple lower bound.  However, given that this upper bound is
tight for $ n \gg k$, we can use this an effective approximation to
better understand the result.  If needed, one can always use the exact
formula.  As can be seen, there is a slight advantage to this bound
over the bounds of the other methods as $r$ increases in size.  This
is due to the fact that this bound takes into account that, for this
method, as one increases the number of corrupt features some of these
features must occur in hypotheses that have already been corrupted.

Next we show that the $k$-subset method is optimal for the case
where all hypotheses have $k$ features.  As explained above, this
shows that $l \geq O(n/k)$ is a tight bound when $n \gg k$.

\begin{theorem}
  If there are $h$ hypotheses where each hypotheses uses at least $k$
  features from $\{x_0,\ldots,x_{n-1}\}$ and $l$ features are corrupt
  then the fraction of corrupt hypotheses is at least
  $1-\prod_{i=0}^{i=l-1} (n-k-i)/(n-i)$.
\end{theorem}

\begin{proof}
  Let $h_i$ be the number of non-corrupt hypotheses after $i$ features
  have been corrupted.  Let $r_i$ be the fraction of hypotheses that
  are corrupted when feature $i$ is corrupted.  We can use the pigeon
  hole principle where we consider the holes the $n-i$ remaining
  features and every hypothesis consists of at least $k$ pigeons.
  Therefore, with $h_i$ hypotheses, there must exist a feature that is
  used by at least $\lceil h_ik/(n-i)\rceil$ hypotheses.  This allows
  the adversary to always corrupt at least $h_ik/(n-i)$ hypotheses.
  This shows that
  \[
    h_{i+1} \leq h_i - \left(\frac{h_ik}{n-i}\right) =
    h_i \left( 1 - \frac{k}{n-i}\right) =
    h_i \left( \frac{n-k-i}{n-i}\right)
  \]
  Given that $h_0=h$, this proves the theorem.
\end{proof}

While it is possible to have a mixture strategy that uses different
numbers of features in the hypotheses, we currently do not see a
compelling reason to use a wide range of sizes for feature subsets.
If one can generate a significant number of hypotheses with few
features that also have good accuracy, then there is no reason to
generate hypotheses using significantly more features that have worse
guarantees.  The only exception is having a difference of one feature
between subsets.  This is close enough to often give a good bound and
is helpful for the fixed-split method and the $k$-modulus method that
will be explained in Section~\ref{MSM}.

\subsection{Random Subspace Method} \label{RSM}

The random subspace method is a well studied method\cite{Ho98} that
takes a random sample of $h$ hypotheses from the set of $C(n,k)$
hypotheses and is therefore an approximation of the $k$-subset
algorithm.  It has been show to give give good performance on many
types of uncorrupt learning problems, but this research is the first
to give provably guarantees in the adversarial setting.  As an
approximation of the $k$-subset method, it is possible to use a
concentration result to show that as $h$ increases with high
probability the method will not select many more than the average
number of corrupt hypotheses from the the $k$-subset algorithm.
However, formally this only holds for a single trial and a strong
adversary will be able to learn which features to corrupt to maximize
the error.  This maximum error can be bounded by explicitly computing
the number of hypotheses that can be corrupted as $l$ increases.
While one can randomly search the space of hypotheses to minimize this
number, in cases of a strong adversary, we recommend using one of the
other algorithms.

\subsection{Modulus Subspace Method} \label{MSM}

We call this the $k$-modulus method as it uses the modulus function to
build a deterministic group of feature subsets where each subset has
$k$ elements.  The $k$-modulus method can provably tolerate as many
corrupt features as the fixed-split method but allows more control
over the parameter $k$.  The technique starts by indexing the $n$
features as $f_0$ to $f_{n-1}$.  Next it creates feature sequences by
modifying the indexes.  Given an sequence of indexes, we add one to
each index using modulo arithmetic.  For example, if $n=10$, $k=3$ and
we start with $(f_0,f_4,f_8)$ then next two sequences would be
$(f_1,f_5,f_9)$ and $(f_2,f_6,f_0)$.  We repeat this procedure $n-1$
times to generate a group containing at most $n$ sequences of
features.  This procedure can be used on all length $k$ index
sequences to create a partition of all $C(n,k)$ feature subsets.
Using every feature subset would make it equivalent to the full
$k$-subset method.  Instead, we propose to use a small number of
groups from the partition.  As we will show, this will control the
cost while still giving us the $O(n/k)$ bound on the number of corrupt
features.

More formally, let $I(b)$ be defined as applying this plus one modulus
operation applied to sequence $b$ of feature indexes.  We define
$G(b)$ to be the set of sequences generated by applying $I(b)$, $n-1$
times on a feature subset $b$.  Let
$G(b) = \bigcup_{i=0}^{n-1} I^i(b)$.  Notice that $I^n(b)=b$ so any
further applications of operator $I$ will repeat previous sequences.

Define $P(n,k)$ to be the partition of all $C(n,k)$ feature subsets
that is generated by applying $G()$ to all possible feature size $k$
sequences.  One way to generate $P(n,k)$ is to iteratively build the
partition elements by applying $G$ to any feature sequence that is not
already part of the incrementally built partition.

To connect the corrupt features to the corrupt hypotheses, we need
some new definitions.  Let $v(b)$ be the binary list that has value 1
if the corresponding indexed element is in $b$, otherwise it has value
0.  For example, $b=(x_0,x_1,x_5,x_6)$ has
$v(b)=(1,1,0,0,0,1,1,0,0,0)$.  Let $p(c,j)$ be a circular right shift
by $j$ of feature sequence binary list $c$.  For example,
$p((1,1,0,0,0,1,1,0,0,0),2)=(0,0,1,1,0,0,0,1,1,0)$.  Finally let
$r(b)$ be the minimal shift in the binary sequence such that
$p(v(b),r(b))=v(b)$.

\begin{lemma} \label{lemma}
  Given a sequence of features $b$ from $\{f_0,\ldots,f_{n-1}\}$, the
  number of elements in partition group $G(b)$ is equal to $r(b)$.
\end{lemma}

\begin{proof}
  Apply operation $I$ a total of $r(b)$ times on $b$.  Given the
  definition of $v(b)$, it must be the case that $I^{r(b)}=b$.
  Therefore the number of elements in the partition group must be less
  than or equal to $r(b)$.

  Assume we apply operation $I$ a total of $c<r(b)$ time on $b$ and
  that $I^{c}=b$.  Based on the definition of $v(b)$ this is a
  contradiction since this implies that $v(b)=v(I^{c})$.  This shows
  that $|G(g)| \geq c$ which proves the lemma.
\end{proof}

We can use this result to relate the number of corrupt features to the
number of corrupt hypotheses.

\begin{theorem}
  Assume the majority vote uses a subset of groups from the partition
  $P(n,k)$.  If an adversary corrupts $l$ features then $rn/k \leq l$
  where $r$ is the fraction of hypotheses that are corrupt.
\end{theorem}

\begin{proof}
  Let $g$ be any of the groups and look at an index set $b$ that
  generates this group.  Create a list of all index sets in this set
  by applying operator $I$ a total of $n$ times; this will include any
  repeated index sets.  Given that we start with $k$ features and
  cycle each index through all $n$ values then each feature index
  occurs $k$ times.  Therefore any corrupt feature will corrupt at
  most $k$ hypotheses.  Given that there are $n$ (potentially
  non-unique) hypotheses in the group, if the adversary corrupts $l$
  features then at most $kl$ hypothesis are corrupt.  This means that
  $r \leq kl/n$.  Based on Lemma~\ref{lemma}, each hypothesis is
  repeated the same number of times.  This shows that $r \leq kl/n$
  even after we remove repeats and therefore $r \leq kl/n$ for group
  $g$.  Given that this result holds for any group $g$, it also holds
  for the union of these disjoint groups which proves the theorem.
\end{proof}

This is the same bound as the fixed-split method.  The main issue
with the $k$-modulus method is that it can be expensive to work
with $n$ hypotheses.  This problem can be partially addressed by
manipulating the structure of the indexing to generate groups that are
smaller than $n$.

\begin{theorem} \label{group}
  For every group in $P(n,k)$, there must exist a common factor $q$ of
  both $n$ and $k$ where the group has size $n/q$.  Furthermore, for
  every $q$ that is a factor of both $n$ and $k$, there must exist at
  least one group with $n/q$ unique elements.
\end{theorem}

\begin{proof}
  Let $b$ be an index sequence.  Based on Lemma~\ref{lemma}, $v(b)$
  must have a structure where the length of the repetition is $r(b)$.
  Given this structure, we can conclude that there exists an integer
  $q$ such that $n=q(r(b))$ where $q$ is the number of times the
  pattern repeats.  Also, we known that binary sequence $v(b)$ only
  has $k$ values that are 1.  Therefore $k=qs$ where $s$ is the number
  of values that are 1 in one of the repeats.  For example, if $n=24$
  and $b=(x_0,x_4,x_7,x_8,x_{12},x_{15},x_{16},x_{20},x_{23})$ then
  then $n=3(8)$ and $k=3(3)$.  Based on Lemma~\ref{lemma}, we can
  conclude that $s=r(b)=n/q$, which proves the first half of the
  result.

  For the second half, assume $q$ is a common factor of both $n$ and
  $k$.  It must always be the case that one can construct a
  non-repeating binary pattern of size $n/q$ where $k/q$ of the
  elements are 1.  For example, we can make the first $k/q$ elements 1
  and the remaining elements 0.  This pattern directly maps to an
  index sequence that, based on Lemma~\ref{lemma}, generates a group
  of size $n/q$.
\end{proof}

\begin{corollary}
  All partition groups in $P(n,k)$ have size $n$ iff $n$ and $k$ are
  relatively prime.
\end{corollary}

\begin{proof}
  Assume all partition groups in $P(n,k)$ have size $n$ and that
  $\mbox{gcd}(n,k) = q > 1$.  Based on Theorem~\ref{group} there must
  exist a group of size $n/q$ which is a contradiction.

  Assume $\mbox{gcd}(n,k) = 1$ and that there is a partition group $g$
  with $|g|<n$.  Based on Theorem~\ref{group}, this is a
  contradiction.
\end{proof}

Various group sizes are relatively easy to generate by creating binary
strings with repeating patterns.  Unfortunately, it is not always
possible to get the exact size we want, but with a slight
modification, we can create a group size that is at most
$\lceil n/k \rceil$.  We accomplish this by adding $k - (n \bmod k)$
dummy features.\footnote{A dummy feature is used to generate the
  groups but is not used with the learning algorithm.}  This adds at
most one dummy feature per hypothesis.  This is similar to the
technique used in the fixed-split method when a perfect split is
not possible.  In that case, we interpreted the split as some
hypotheses having one less feature, but that is equivalent to adding
dummy features. In fact, when the group size is $\lceil n/k \rceil$,
this is the fixed-split method which shows that the modulus method is
a generalization of that simpler technique.

\section{Decreasing Majority Vote Cost} \label{Ensemble-Cost}

One issue with using an large ensemble of hypotheses is prediction
time.  We recommend speeding up prediction time by using sequential
sampling techniques that randomly sample the voting hypotheses until a
prediction can be made with a controllable probability of
correctness~\cite{Wald47}.  This is most beneficial when the adversary
only occasionally corrupts an instance since, on many problems,
uncorrupt instances are often predicted correctly for a large fraction
of the hypotheses and sequential sampling can quickly and accurately
estimate the majority label.  In addition, while most sequential
sampling techniques assume that sampling is done with replacement, it
should be possible to get better bounds since sampling in this case is
done without replacement.  Another refinement is to take advantage of
a multi-core computer architecture and evaluate several hypotheses in
each step of the sequential sampling.

\section{Experiments} \label{Experiments}

In this section, we give experiments on a range of datasets which
include worst case bounds when $l$ features are corrupt and actual
results against a simple weak adversary.  The weak adversary is not
supposed to model a worst case adversary, but is primarily used to
show how traditional algorithms degrade with a simple change in the
distribution.

To generate our weak adversary, we want to corrupt different relevant
features on each trial.  We do this by creating a distribution on the
features using mutual information (MI)~\cite{Kul59}.  For each
instance, after selecting $l$ features to corrupt from the created MI
distribution, we change the feature value to be the maximum value on
the opposite side of the mean.  For example, if a feature has a range
of [-2,3] and a mean of 1 in the training data then during testing we
corrupt a feature with value -1 by changing it to 3.

The eight datasets we used are described in
Table~\ref{learning-problems}. All our voting techniques use random
forest as the basic classifier as it is easy to tune to give high
accuracy on most machine learning problems~\cite{CNM06}.  All graphs
include a 99\% confidence interval based on the exact binomial
distribution~\cite{Langford05}. We used the scikit-learn Python
library \cite{scikit} for all our code and used the standard cross
validation library for parameter selection where the number of random
features select was from $\{\sqrt{n},2\sqrt{n},n\}$ and the number of
trees was selected from $\{10,20,50,100\}$.  For all experiments, we
used a 80/20 train/test split.  The only exception is that we set
100,000 as the maximum number of instances for training or testing.
For all datasets, we combined and permuted the existing data.  This
was to ensure that all the data was iid before the adversary corrupts
the instances.  For the data dependent bounds, the results for the
random subspace method are the expected bounds give that the 500
hypotheses are sampled from all possible subsets.  On all experiments
we also report the error rate of predicting with the majority label.
This is poor but reasonable default classifier since it gives
information about label skew and is unaffected by feature corruption.
All experiments were performed on a 44 core Intel Linux server.

Given the large number of experiments and independent variables, we
limit the presented results to a single number of hypotheses for both
the fixed split technique and the random subspace technique.  For the
fixed-split method we picked the best result from
$h = \{3,5,7,9,11,13,15,21,31\}$.  For feature corruption, we ran
experiments from $0$ to $35$ corrupt features.  We sometimes stopped
the experiments early if error-rate degraded to majority label
baseline.  For the random subspace, given the expense of running the
algorithm, we only used a single set of parameters.  To control cost,
we used $h=500$ and set $k$ based on the $n/k$ value that gave the
best results for the fixed-split method.\footnote{When focused
  on a single problem with sufficient validation data, we recommend
  testing more parameters.}  Other values of $h$ and $k$ give
qualitatively similar results when the features are significantly
corrupted.  We do not report results for the modulus subspace method
as we are still in the preliminary stage of evaluating this method.

\begin{table*}[t]
  \caption{Datasets used in corruption experiments.}
  \label{learning-problems}
  \begin{center}
    \begin{small}
      \begin{sc}
        \begin{tabular}{lcccccc}
          \toprule
          Data set & Features & Label & $|Y|$ & $|$Train$|$ & $|$Test$|$ & Access \\
          \midrule
          UNO      & 1024 & Device Mode & 11 & 24413 & 6104 & non-public \\
          Pi       & 1024 & Device Mode & 3 & 6067 & 1517 & non-public \\
          Smart    & 512 & Device Mode & 4 & 10000 & 2000 & non-public \\
          Character Font & 409 & Italic Arial & 2 & 20989 & 5248 & UCI \\
          IoT Botnet & 115 & ACK Attack & 2 & 100000 & 42591 & UCI \\
          UJIIndoorLoc & 520 & Building Floor & 5 & 16838 & 4210 & UCI \\
          US Census Data & 68 & Marital Status & 5 & 100000 & 100000 & UCI \\
          CoverType  & 54 & Forest Covers & 7 & 100000 & 100000 & UCI Repo \\
          \bottomrule
        \end{tabular}
      \end{sc}
    \end{small}
  \end{center}
\end{table*}

\subsection{Electromagnetic Side Channel Data} \label{LADS-Experiments}

We are currently working on a project that determines the
computational mode of a device based on its unintended electromagnetic
(EM) emissions\cite{CASPER18}.  The goal of this project is to
determine if unauthorized code is running on the device.  While we do
not have space to give all the details on this problem, we have
captured EM data of two devices while they execute authorized code.
The data is captured using an antenna and a software defined radio
sampling at 25 MHz at a specified central frequency.  We then
processed that data using the fast Fourier transform into 1024
frequency bins and use those energy levels as our features.  Through
various techniques, we have labeled this data into device modes.

One convenient property of this data is the presence of harmonics
where information is correlated over the frequency spectrum.  This
along with strong differences between the device modes, makes the
learning problem fairly easy for standard machine learning techniques.
The difficulties arise when the system is used to make label
prediction at a different time and/or location.  Changing,
intermittent EM noise can be present and can corrupt different
features over a sequence of predictions.  This was our motivation to
develop these techniques.  While this EM noise problem is not a worst
case adversary, it has properties that make it difficult to analyze
with traditional train/test assumptions.

We ran experiments for three devices: an Arduino UNO, a Raspberry Pi,
and a smart meter.  It was difficult to do controlled experiments with
real environmental noise, so we used the adversarial noise model
explained at the start of this section.

The left side of Figure~\ref{em-plot} gives the result for an Arduino
UNO running a simple program consisting of loops of NOP statements.  A
simple loop that repeats a specific number of clock cycles causes a
repeated behavior that is picked up at a specific frequency as an
amplitude modulation of the CPU clock.  As can be seen, the random
forest eventually gets close to half the predictions wrong.  The
fixed-split method does much better, but the best results are for the
random subspace method.  This is reasonable as it has the best bound
against a weak adversary as explained in Section~\ref{RSM}.

We also give graphs of the worst case adversarial bounds.  It is not
surprising that these bounds are much worse since they assume the
adversary can maximize the errors by controlling the prediction of any
corrupt hypothesis.  Still the bounds are somewhat positive given that
the $\frac{n}{k}$ ratio is roughly 13.  The voting hypotheses must
have high accuracy to be able to tolerate corruption of almost half
the hypotheses.

\begin{figure}[ht]
  %\vskip -0.15in
  \begin{center}
    \includegraphics[scale=0.15]{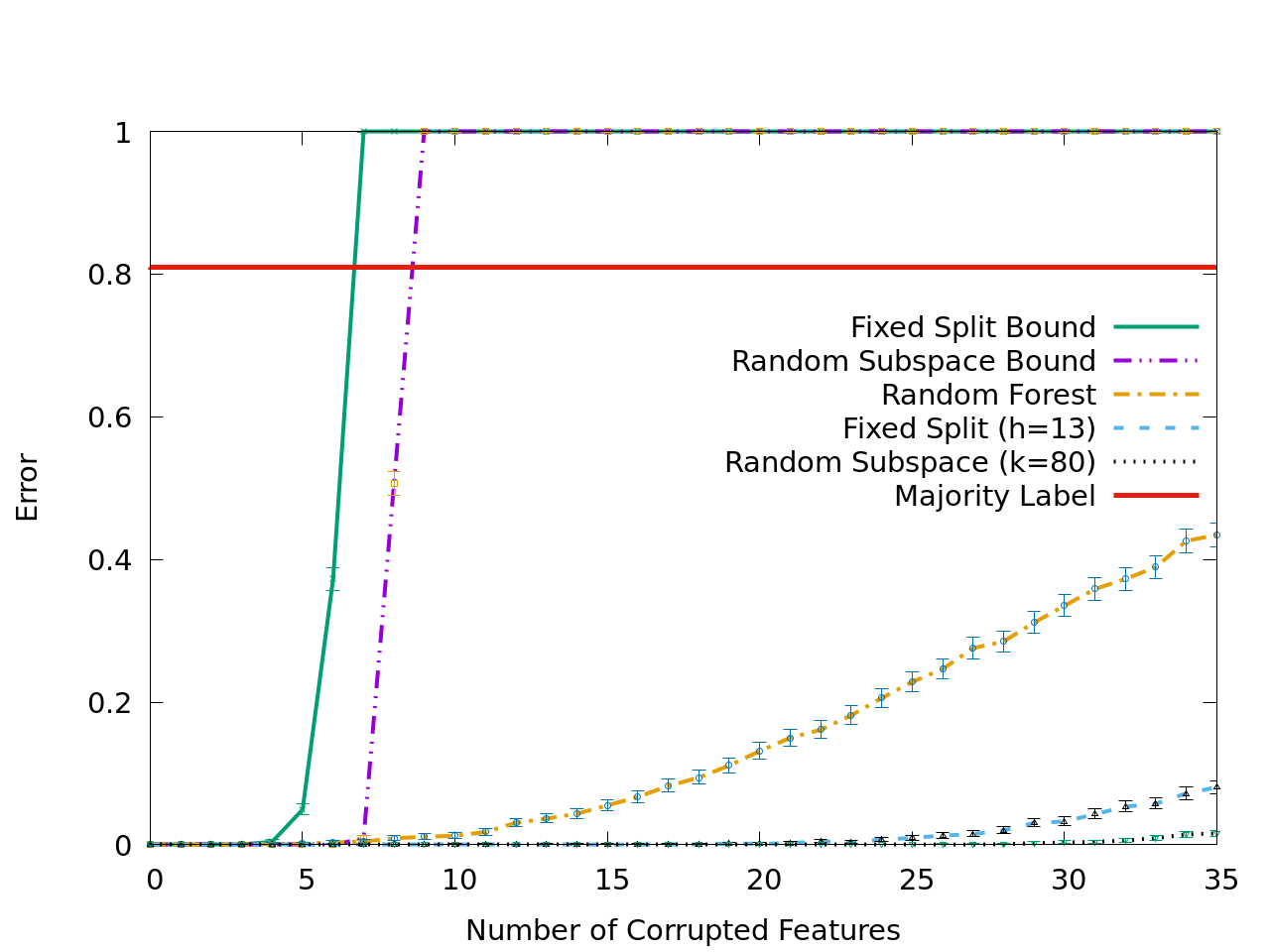}
    \includegraphics[scale=0.15]{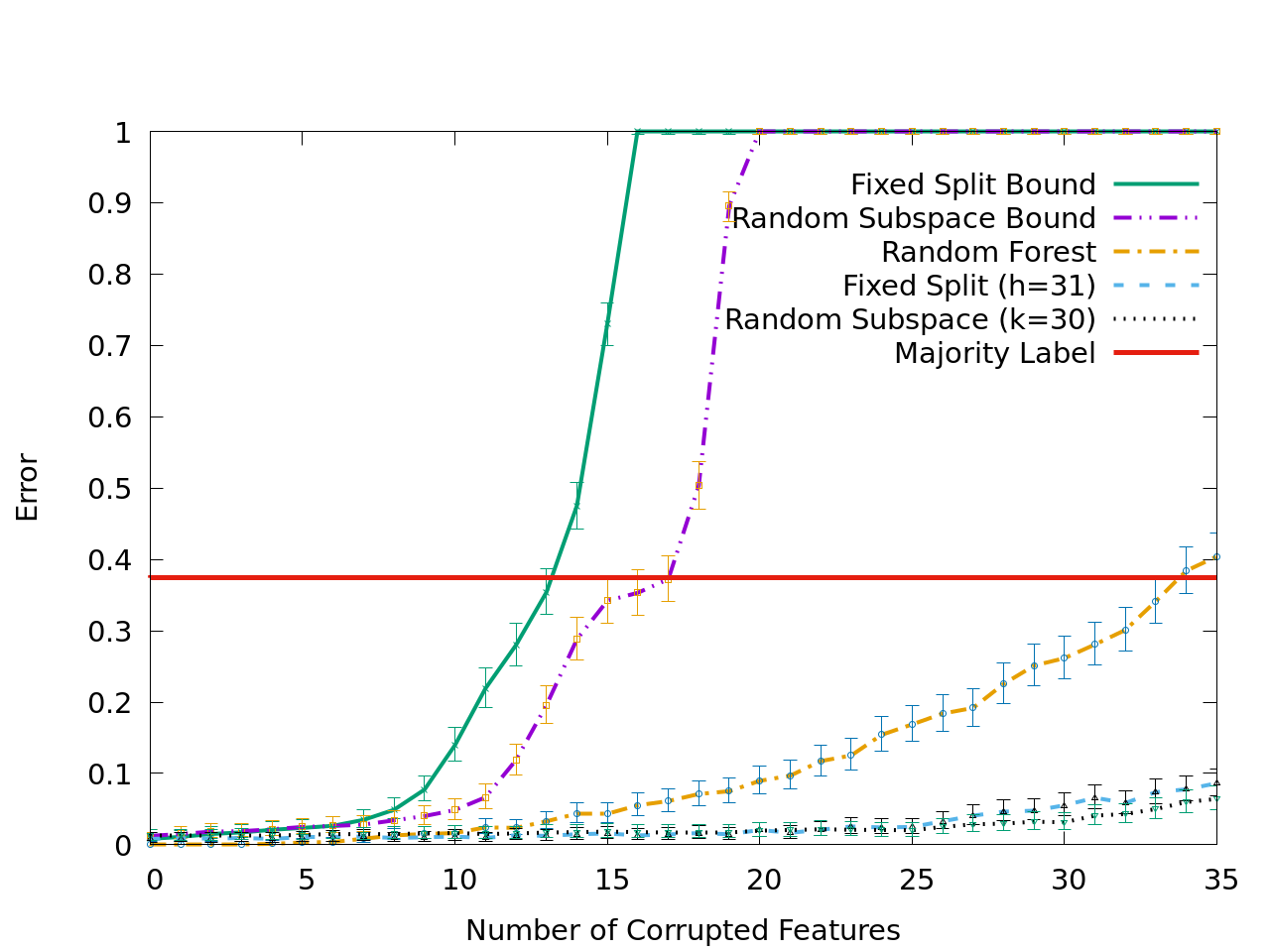}
    \caption{EM experiments with UNO on the left and Raspberry Pi on the right}
    \label{em-plot}
  \end{center}
  %\vskip -0.2in
\end{figure}

The right side of Figure~\ref{em-plot} tells a similar story.  It is
based on a Raspberry Pi running Linux with a simple program that loops
over SHA, string search, and sleep.  While hard to see, the random
forest classifier is doing slightly better at the start; however it
quickly decays as $l>10$.  Again the best performance, as the weak
adversary corruption increases, is the random subspace method.

\begin{figure}[ht]
  \begin{center}
    \includegraphics[scale=0.15]{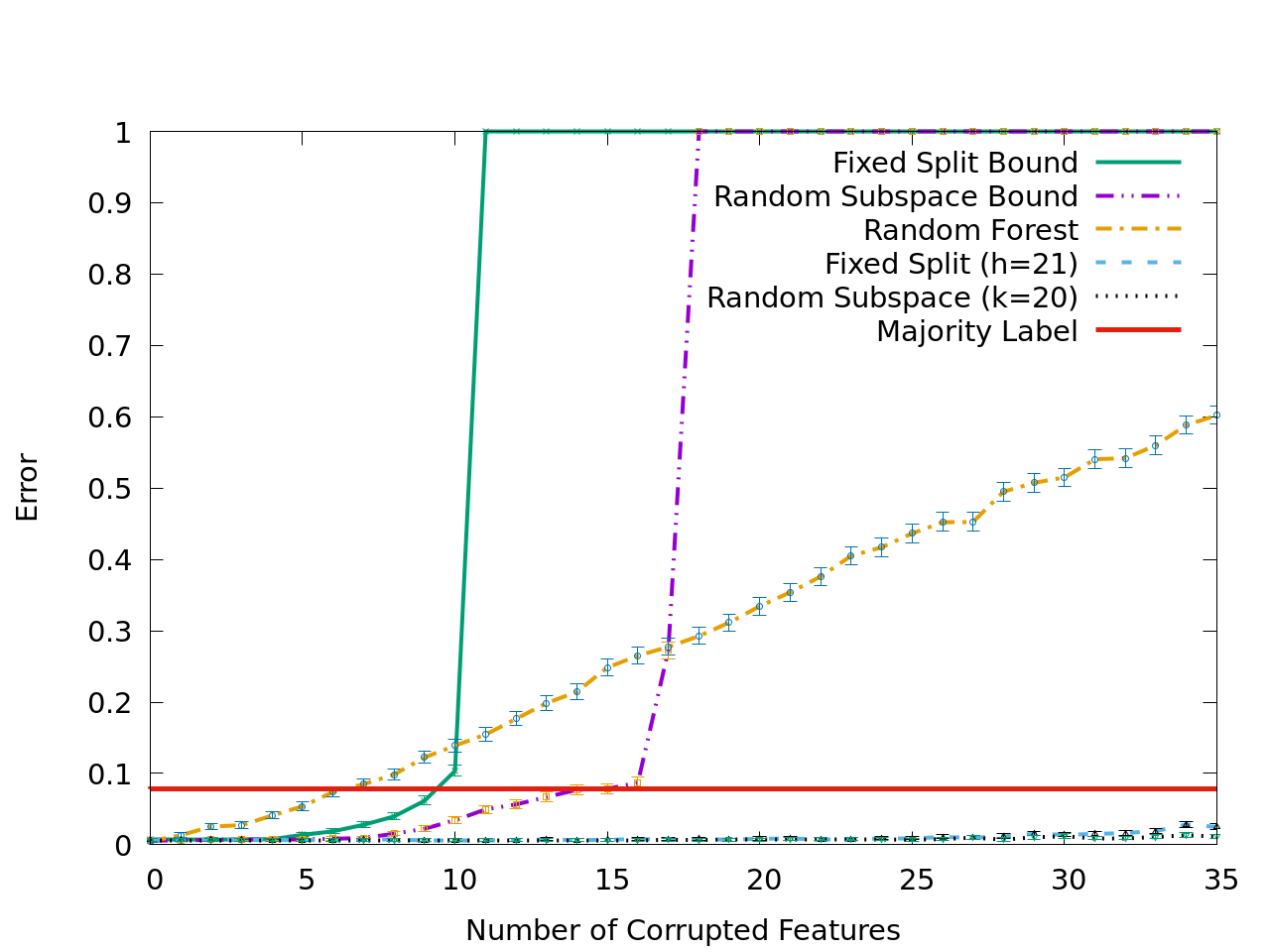}
    \includegraphics[scale=0.15]{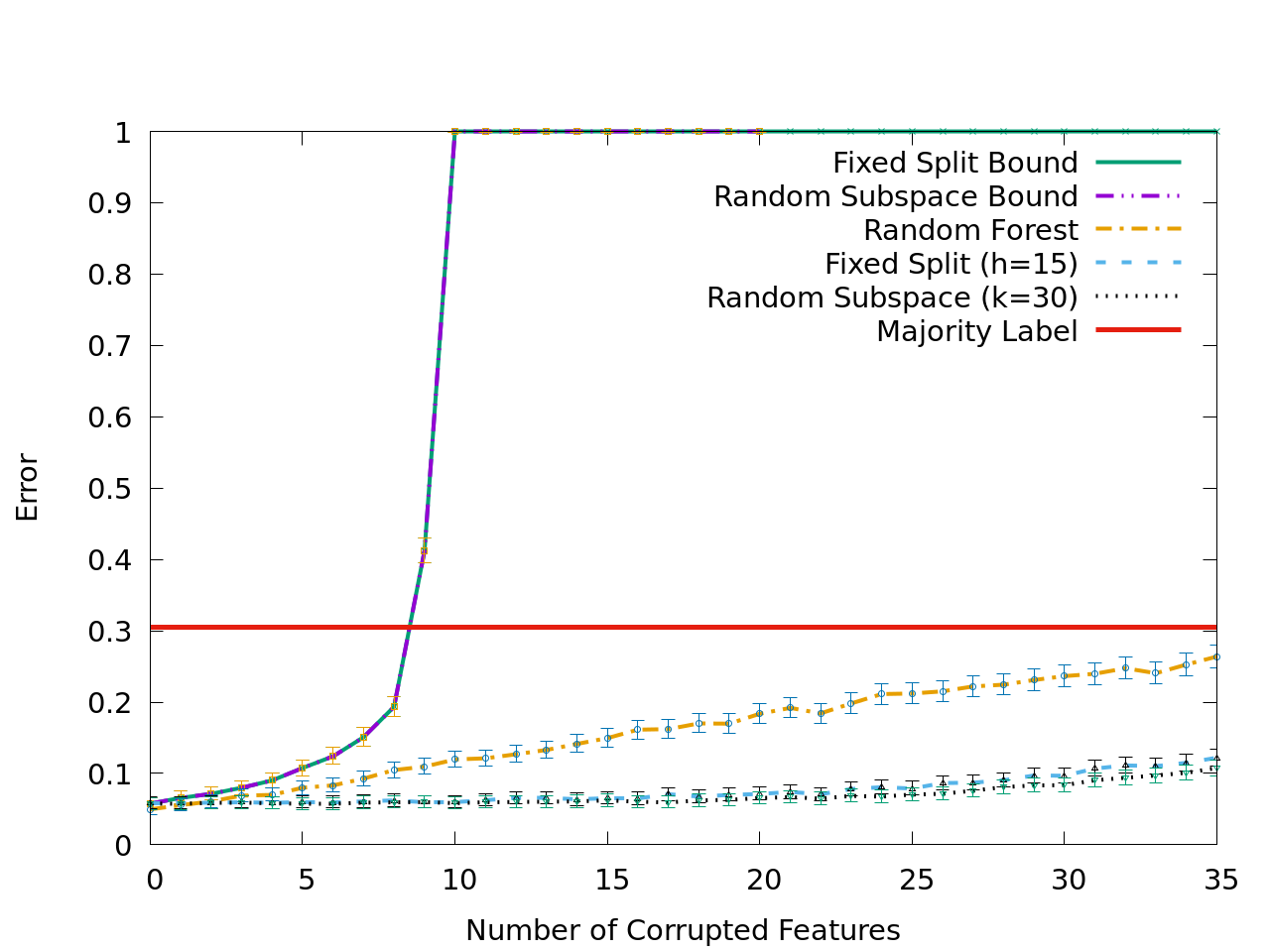}
    \caption{EM smart meter experiment on the left and UCI font experiment on the right}
    \label{mixed-plot}
  \end{center}
\end{figure}

On the left side of Figure~\ref{mixed-plot}, we give the results of
the smart meter experiments based on unmodified firmware running on
the device.  We are unsure why the result are so much worse for random
forest; perhaps it is related to the label skew in this problem.
However, both fixed-split and random subspace are largely resistant to
the corruption with the random subspace method having the lowest
error.

\subsection{UCI Data} \label{UCI-Experiments}

We selected five UCI datasets \cite{DKT18}, described in
Table~\ref{learning-problems}, from UCI by sorting based on number of
instances and choosing problems that fit certain criteria.  In
particular, we selected classification problems but avoided any
problems with less than fifty features or that required extensive
feature processing.  We also avoid problems that contained features
that were clearly a function of some set of original features.  The
motivation for our problem is that the basic features are
independently susceptible to noise, and having features that are
functionally related to each other would break that assumption and
spread the corruption.  We suggest that any derived, functionally
related features be placed in the same voting hypothesis to avoid
spreading corruption.

Our first UCI problem tries to determine if a bit-mapped Arial font
character is italic. The results here are similar to the previous
section with the exception that all the classifiers do not start at
zero error.  As can be seen in the right side of
Figure~\ref{mixed-plot}, the error increase for both subspace
classifiers is very slow.  At $l=35$ both classifiers are doing much
better than random forest with a slight advantage to the random
subspace classifier.

\begin{figure}[ht]
  \begin{center}
    \includegraphics[scale=0.15]{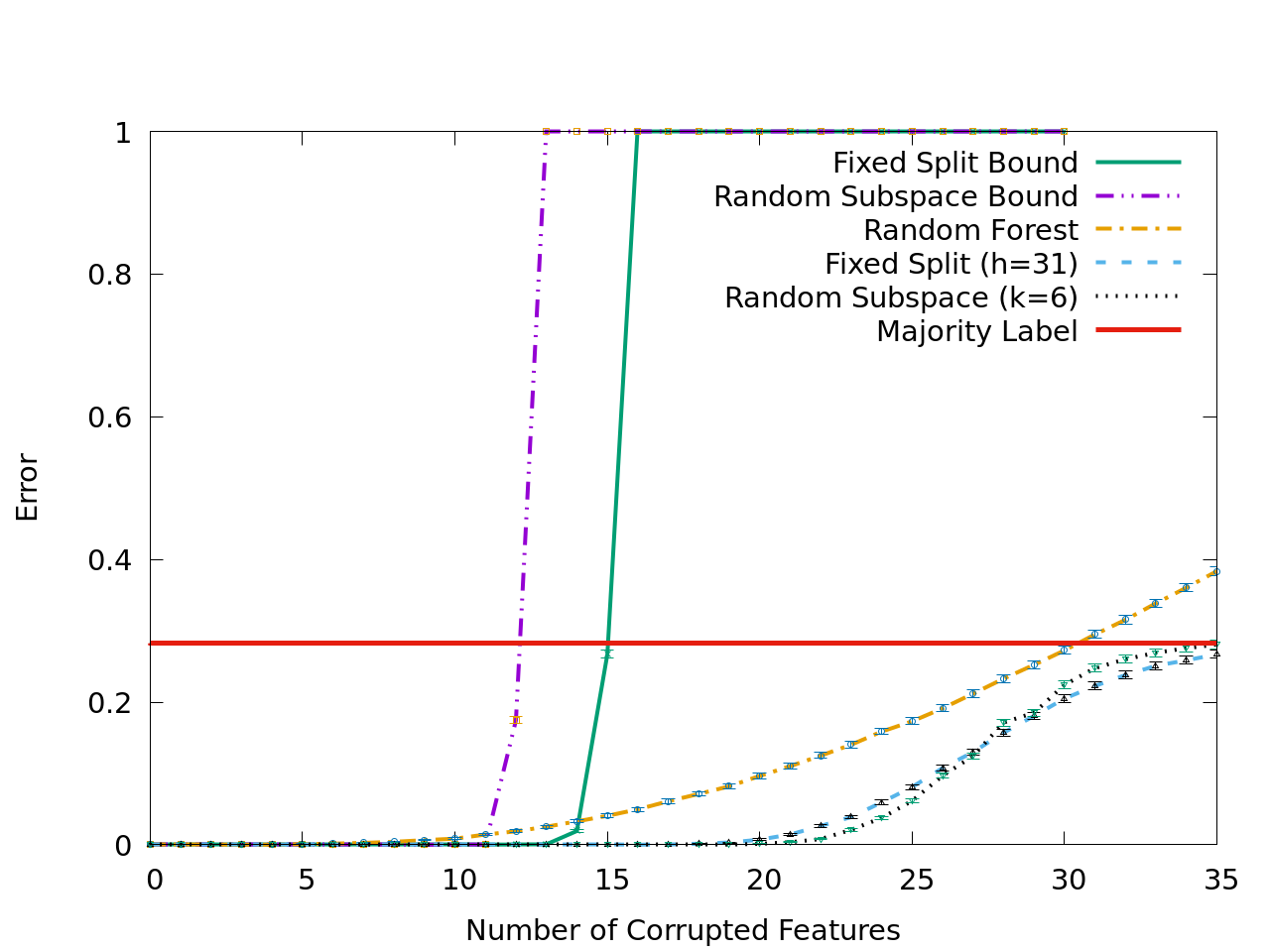}
    \includegraphics[scale=0.15]{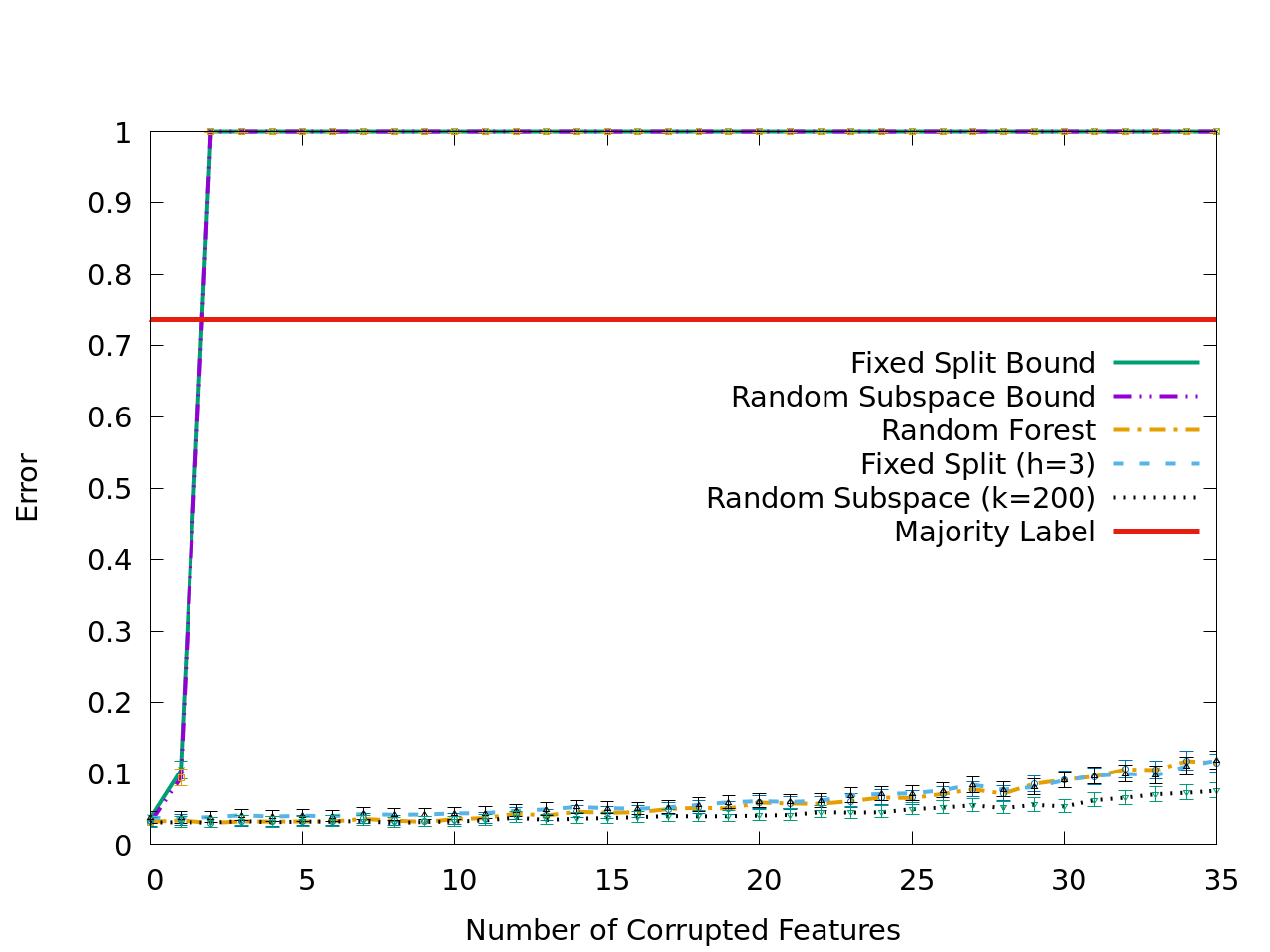}
    \caption{UCI botnet experiment on the left and UCI indoor location experiment on right}
    \label{uci1-plot}
  \end{center}
\end{figure}

On the left side of Figure~\ref{uci1-plot}, we give the results for
the binary label problem of determining whether a botnet attack is
occurring~\cite{MBM18}. Again, the results are positive with the
subspace classifiers tolerating roughly twice as many corrupt features
before making a significant number of errors.  After $l=20$, both
subspace classifiers start to rapidly decay.  While it is possible
that increasing $h$ and reducing $k$ could decrease the error rate
with these large amounts of corruption, the classifiers are already at
the point of having only six features per hypotheses which seems
extreme.  It is likely this problem has a large amount of relevant
feature redundancy.

The results in right side of Figure~\ref{uci1-plot} are interesting
and are based on predicting the floor the user is on in a building
based on data collected from Wi-Fi access points~\cite{TMM14}. The
adversary only has a minimal effect on all the classifiers.  We are
unsure if this is caused by the feature sparsity of this problem
combined with our choice of adversary.  We plan to study the issue of
instance sparsity in the future.  However, even in this case, we do
see a decrease in error rate for the random subspace algorithm.

\begin{figure}[ht]
  %\vskip -0.15in
  \begin{center}
    \includegraphics[scale=0.15]{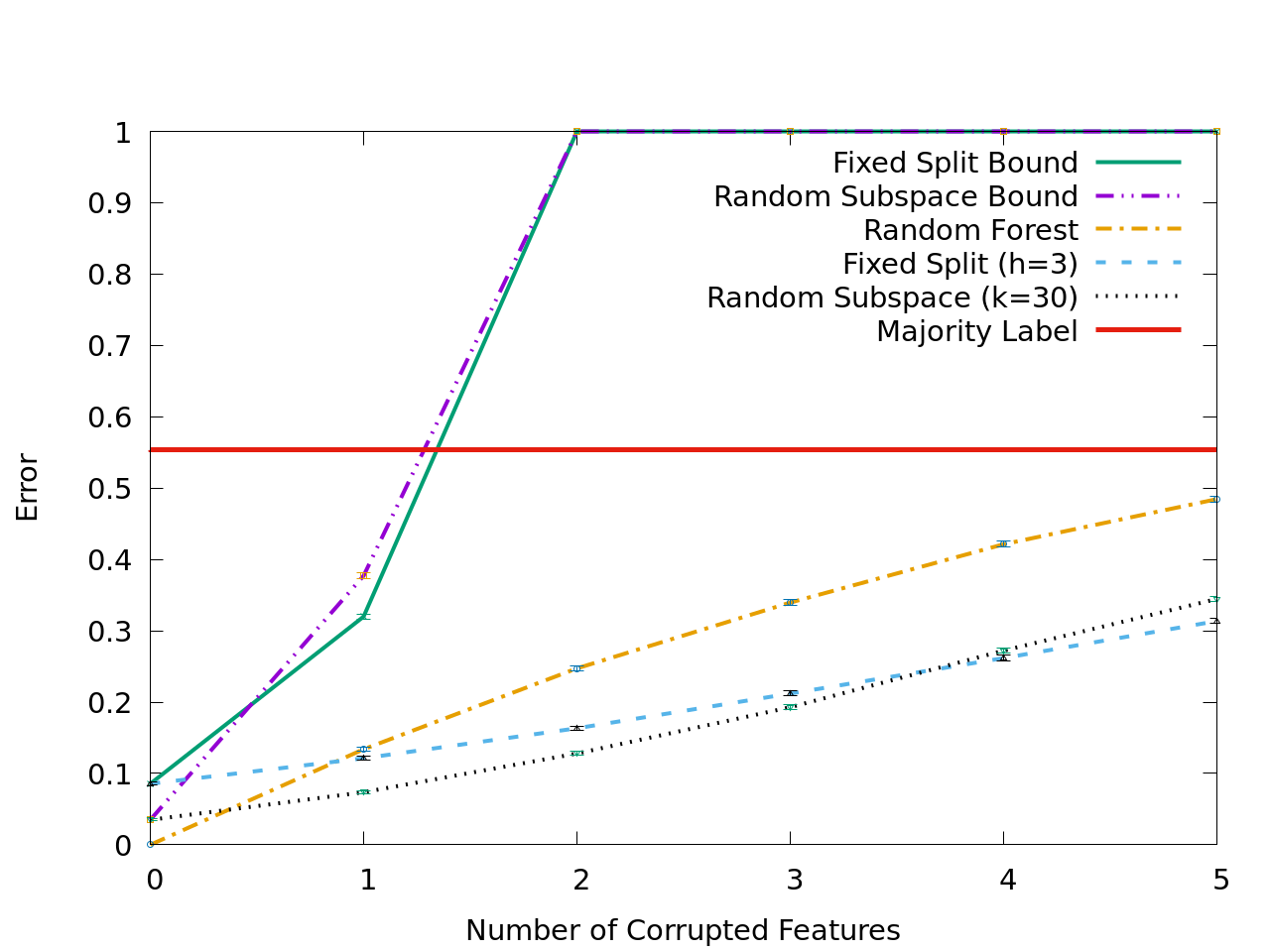}
    \includegraphics[scale=0.15]{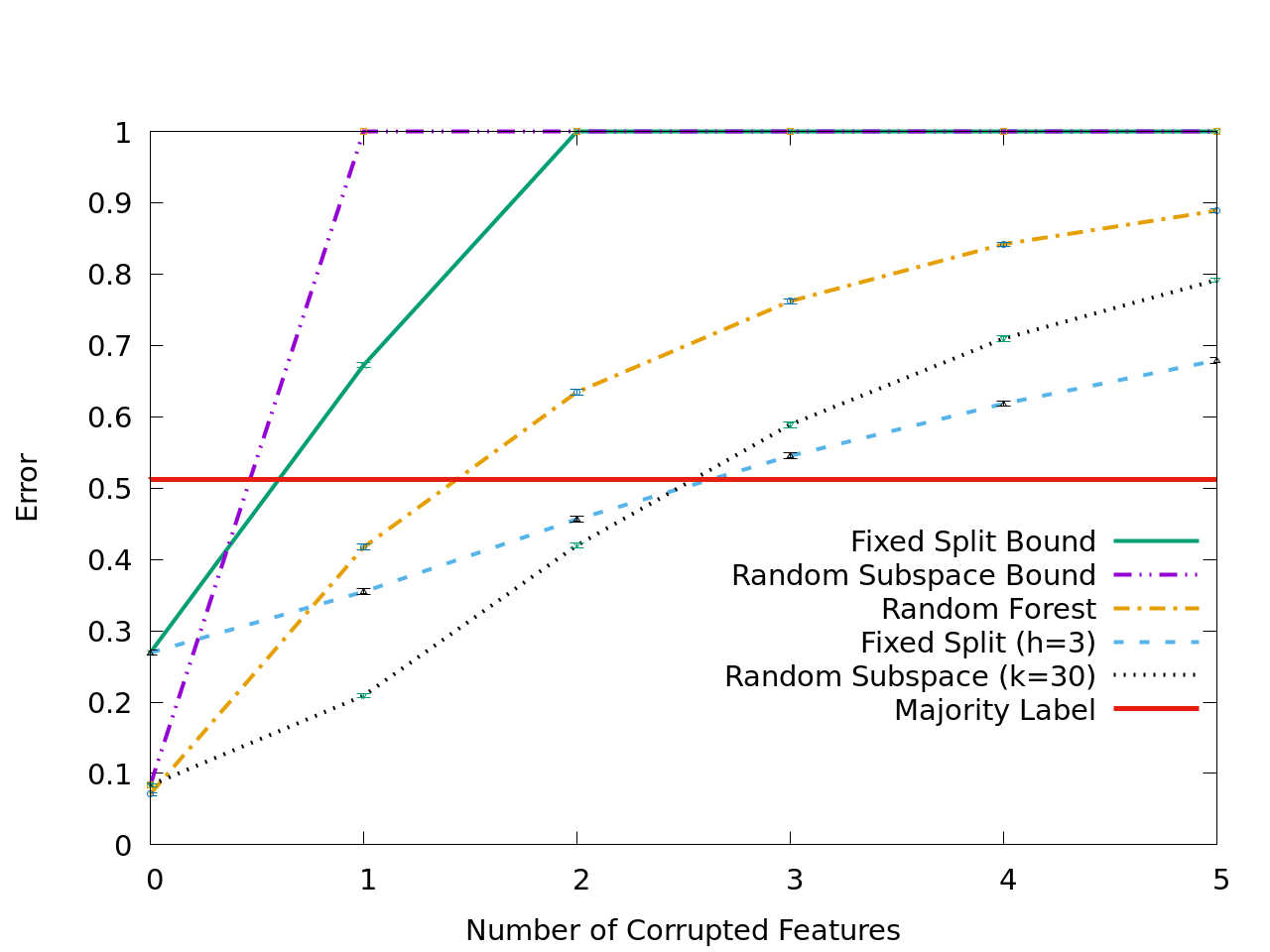}
    \caption{UCI census experiment on the left and UCI forest cover type on the right}
    \label{uci2-plot}
  \end{center}
  %\vskip -0.15in
\end{figure}

Our next data set is based on Census data.  Here we defined the label
based on the marital status since it has five labels values and
reasonable class balance.  As can be seen in the left side of
Figure~\ref{uci2-plot}, we cannot tolerate as many corrupt features,
but we also have significantly fewer total features.  Also we suffer
some loss of accuracy in the non-corrupt case, but as soon as the
corruption starts, the subspace techniques have lower error rates than
random forest.  In addition, we use the small value of $h=3$ as the
non-corrupt error rate rises quickly as $h$ increases.  This shows the
difficulty of the non-corrupt form of this problem for the subspace
techniques.  In principle, subspace method will not work with all
machine learning problems.  We will address this in the next section,
but in general, multiple algorithms should be tested with validation
datasets.

On the right side of Figure~\ref{cover-plot}, the learning problem
labels different types of forest cover.  The results are similar to
the previous Census data; however, in this case the problem is even
more difficult to learn.  Still, we do see improvement for $l=1$ and
$l=2$.  For higher values of $l$ all the algorithms are doing worse
than majority label prediction algorithm.

\section{Difficult Target Functions} \label{Difficult-Functions}

An alternative way to interpret our results is a way to quantify what
types of learning problems cannot be solved by subspace methods.  For
example, a conjunction with no redundant attributes will need at least
half the majority vote hypotheses to have every relevant variable.
Our subspace methods are designed to make sure any $l$ features are
missing in more than half of the hypotheses.  For a conjunction with
three terms, the chance that all three will appear in more than half
the hypotheses will be small even when $k$ is fairly large.  For
example, with $n=1000$ one would need to select $k > 793$ features for
the subspace method to have a 0.5 chance of working even without
corruption.  It is interesting how well many of our UCI experiments in
Section~\ref{UCI-Experiments} perform with much bigger $n/k$ ratios.
While not definitive, this suggests there are large amounts of
relevant feature redundancy in many practical problems.  This is
related to the fact that a sufficiently strong adversary can make
learning impossible by corrupting instances to a part of the space
with a different label.  On concepts like the previously mentioned
conjunction, without redundancy high accuracy can be impossible to
achieve against a zero norm adversary even when not using subspaces
since the adversary can change one feature to change a true
conjunction to a false conjunction.  This suggests a connection
between the performance of subspace methods on standard iid batch
learning and the performance of machine learning against zero norm
adversaries since both cases exploit redundancy.  It also suggest
algorithms that attempt to minimize the number of relevant features
are more susceptible to zero norm adversaries since they might learn
functions that remove redundancy.

\section{Future Work}

One interesting extension of this work is to apply it to regression
problems where prediction is a real number.  In this setting, we
replace majority vote with the median of the of ensemble predictions.
Surprisingly, all of our results carry over to this setting with only
a small increase in the computational complexity of computing the data
dependent bounds.  The key insight is that by using the robust median
statistic, the damage an adversary can do is limited.  The worst an
adversary can do is to corrupt hypotheses on one side of the median
and shift them to an extreme value on the other side of the median.
This will maximize how much the median changes.  The actual amount of
change depends on the number of corrupted hypotheses and the
non-corrupt empirical distribution of predictions.  At a minimum, at
most 50\% of the hypotheses can be corrupt otherwise the median can be
changed to any value.  This is equivalent to the situation with
majority vote classification.  The main difference with our results on
classification is the extra difficulty in generating the bound.  Since
we no longer have a binomial distribution for the loss function, we
need to make assumptions about the distribution of errors for the
uncorrupt hypotheses in order to generate bounds.  This is typical for
the regression setting and not an additional weakness of the
adversarial analysis.  In practice, one can use an uncorrupt test set
to estimate the distribution of uncorrupt error and to bound the error
as a function of the number of corrupt hypotheses.

%Look into generating worst case bounds by considering the worst
%possible distribution.  Is this viable?  Probably not.  Might need to
%give delta failure rates to avoid really crazy distributions.  Think
%Markov and Chebechev inequality.

\section{Conclusion}

This paper presents new subspace methods along with a new analysis
that shows, with appropriate parameters, subspace methods can tolerate
arbitrary corruption of a limited number of features.  While the
amount of corruption that can be tolerated depends on unknown details
of the problem, we give a statistic that can be used to estimate the
worst case performance using uncorrupt test data.  This is similar to
the traditional test bound used in iid supervised learning but allows
us to extend that framework to handle adversarial changes in the
instances.  While adversaries are not typically encountered in
learning problems, the proofs also apply to other situations that
include various types of instance distribution drift.  We give
experiments to show these algorithms perform well on a range of
realistic problems including five UCI datasets and three new datasets
based on electromagnetic side channel information.

\bibliography{learning}
\bibliographystyle{abbrv}

\end{document}